\documentclass[twoside]{article}

\usepackage[accepted]{aistats2026}
\usepackage[round,authoryear]{natbib}
\usepackage{mathtools, amsmath, amsfonts, amssymb,amsthm,multicol,derivative,enumitem,graphicx,float,tikz,tikz-qtree,forest,verbatim,bm,booktabs,lipsum}
\usepackage[ruled,vlined,linesnumbered]{algorithm2e}
\usetikzlibrary{fit,bayesnet, positioning, quotes,calc}
\usepackage[makeroom]{cancel}
\usepackage{stfloats}
\usepackage{adjustbox}
\usepackage[hidelinks,breaklinks=true]{hyperref}
\usepackage{xurl}
\usepackage{multirow}
\usepackage{subcaption}  
\usepackage{placeins}    
\usepackage{float}       
\usepackage{booktabs}
\usepackage{tabularx}
\usepackage{array}
\newcolumntype{Y}{>{\raggedleft\arraybackslash}X} 

\usepackage[nameinlink,capitalize,noabbrev]{cleveref}

\theoremstyle{definition}

\theoremstyle{definition}

\theoremstyle{definition}

\theoremstyle{definition}

\theoremstyle{definition}

\theoremstyle{definition}

\theoremstyle{definition}

\theoremstyle{definition}

\definecolor{lightblue}{RGB}{173,216,230} 
\definecolor{lightgray}{gray}{.97}

\let\oldforall\forall
\renewcommand{\forall}{\oldforall \, }
\let\oldLeftrightarrow\Leftrightarrow
\renewcommand{\Leftrightarrow}{\oldLeftrightarrow \,\, }
\let\oldexist\exists
\renewcommand{\exists}{\oldexist \: }

\newcommand{\KL}{D_{\mathrm{KL}}}

\newcommand{\indep}{\perp \!\!\! \perp}

\newcommand{\E}{\mathbb{E}}

\newcommand{\bfx}{\mathbf{x}}

\newcommand{\bfz}{\mathbf{z}}

\newcommand{\bfy}{\mathbf{y}}

\newcommand{\bfv}{\mathbf{v}}

\newcommand{\bfa}{\mathbf{a}}

\newcommand{\cD}{\mathcal{D}}

\newcommand{\cT}{\mathcal{T}}

\usepackage[most]{tcolorbox}   
\usepackage{xcolor}            

\definecolor{bubblefill}{HTML}{FAFAD2} 
\definecolor{bubbleframe}{HTML}{FFB84D} 

\newtcolorbox{contribbox}[2][]{%
	colback=bubblefill,
	colframe=bubbleframe,
	coltitle=black,
	boxrule=0pt,
	arc=2mm,
	left=6pt,right=6pt,top=6pt,bottom=6pt,
	fonttitle=\bfseries,
	title={#2},
	#1
}

\usepackage{xcolor}
\definecolor{darkblue}{HTML}{B8860B}
\usepackage{pifont}     

\definecolor{lightblue}{RGB}{173,216,230} 
\definecolor{lightgray}{gray}{.97}

\newtheorem{proposition}{Proposition}

\newtheorem{theorem}{Theorem}
\newtheorem{remark}{Remark}

\begin{document}

%

%

\twocolumn[

\aistatstitle{Adversary-Free Counterfactual Prediction via Information-Regularized Representations}


\aistatsauthor{
  Shiqin Tang$^{*,a}$ \And
  Rong Feng$^{*,a,b}$ \And
  Shuxin Zhuang$^{a,b}$ \And
  Youzhi Zhang$^{a}$ \And
  Hongzong Li$^{\dagger,c}$
}


\aistatsaddress{
  $^{a}$ CAIR-CAS \And
  $^{b}$ CityUHK \And
  $^{c}$ HKUST
}]

\begin{abstract}
We study counterfactual prediction under assignment bias and propose a mathematically grounded, information-theoretic approach that removes treatment–covariate dependence without adversarial training. Starting from a bound that links the counterfactual–factual risk gap to mutual information, we learn a stochastic representation $Z$ that is predictive of outcomes while minimizing $I(Z;T)$. We derive a tractable variational objective that upper-bounds the information term and couples it with a supervised decoder, yielding a stable, provably motivated training criterion. The framework extends naturally to dynamic settings by applying the information penalty to sequential representations at each decision time. We evaluate the method on controlled numerical simulations and a real-world clinical dataset, comparing against recent state-of-the-art balancing, reweighting, and adversarial baselines. Across metrics of likelihood, counterfactual error, and policy evaluation, our approach performs favorably while avoiding the training instabilities and tuning burden of adversarial schemes.
\end{abstract}

\begingroup
\renewcommand\thefootnote{}
\footnotetext{* Equal contribution. $\dagger$ Correspondence to: Hongzong Li \texttt{\textless lihongzong@ust.hk\textgreater}.}
\endgroup

\section{INTRODUCTION}

Causal effect estimation addresses the ``what if'' (counterfactual) question:
\emph{What outcome would have occurred for a patient, possibly contrary to fact, had the patient received a different treatment?}
In real-world (observational) data, however, who receives a treatment is rarely random. For example, clinicians may preferentially administer a medication to patients with more severe symptoms, while patients with milder illness are less likely to receive it. If we then compare outcomes between the treated (on average sicker) and untreated (on average healthier) groups, the difference reflects both the medication’s effect and the baseline severity gap; as a result, the effect we would predict for mildly ill patients can be biased. This systematic distortion—arising because treatment choice depends on prognostic factors—is often called \emph{assignment bias} (a form of confounding) \citep{rosenbaum1983central, robins1986, robins2000marginal}. 

\textbf{Existing solutions.} A large body of work mitigates such bias via balanced representations \citep{johansson2016,shalit2017}, adversarial domain-invariance \citep{ganin2016domain, kazemi2024adversarially}, and reweighting or orthogonalization \citep{hainmueller2012entropy,huang2007correcting,chernozhukov2018double}. Yet, adversarial approaches often introduce instability and substantial tuning overhead, and balancing via integral probability metrics can become brittle in complex, high-dimensional settings.

\textbf{An information-theoretic lens.}
We offer a unifying perspective that explains why these strategies help. The key idea is to view bias reduction through the amount of information covariates carry about treatment choice. When the features of an individual strongly reveal which treatment they received, treated and untreated groups can look systematically different, and naive comparisons are unreliable. By contrast, when a learned representation removes or attenuates this treatment-revealing signal—while still keeping what matters for predicting outcomes—treated and untreated cases become more comparable, and effect estimates improve. This information-theoretic viewpoint clarifies the goals underlying balancing and domain invariance and provides principled guidance for how to build representations that support counterfactual prediction.

\textbf{Our approach in a nutshell.}
We introduce an information-theoretic representation learning method that explicitly discourages treatment-revealing features without relying on adversarial training. Concretely, we learn a stochastic representation that is predictive of outcomes yet carries little mutual information with treatment assignment. This direct control of the dependence between representation and treatment provides a simple, stable alternative to adversarial objectives, and it can be seamlessly extended to the dynamic case with time-varying covariates. We refer to the static model as the \textbf{S}tatic \textbf{I}nformation-theoretic \textbf{C}ounterfactual \textbf{E}stimator (\emph{SICE}) and its sequential extension as the \textbf{D}ynamic \textbf{I}nformation-theoretic \textbf{C}ounterfactual \textbf{E}stimator (\emph{DICE}).


\paragraph{Paper structure.}
Section \ref{sec:bg_causal_ml} lays the foundation by formalizing the setting and assumptions, and develops risk–gap bounds that motivate our approach.
Section \ref{sec:sice} and \ref{sec:dice} respectively introduce the proposed \emph{SICE} and \emph{DICE} models. 
Section \ref{sec:exp} presents experiments on controlled simulations and a real clinical dataset, followed by discussion and limitations. 
Supplementary proofs and additional experimental results are deferred to the appendix.

\begin{contribbox}{}
\textcolor{darkblue}{\textbf{Contributions.}} \textcircled{\scriptsize 1} We present an information-theoretic formulation that unifies balancing, reweighting, and invariance by linking the risk gap to treatment–covariate dependence.
\textcircled{\scriptsize 2} We propose two models -- \emph{SICE} (static) and \emph{DICE} (sequential) -- that avoid adversarial training by directly approximating the mutual-information term.
\textcircled{\scriptsize 3} In both controlled simulations and real data, we vary treatment dimensionality from small to ultra–high and show the proposed models remain competitive as dimensionality grows.
\end{contribbox}

\section{BACKGROUND}
\label{sec:bg_causal_ml}



We first fix notation and explicitly state the assumptions and conditions we rely on, so that the learning problem is pinpointed precisely and the regimes in which our method applies are rigorously defined.

\paragraph{Notations.} We denote covariates (or confounders) by $X$ (e.g., a patient’s profile), treatment by $T$, and outcome by $Y$, with $X \in \mathcal{X}$, $T \in \mathcal{T}$, and $Y \in \mathcal{Y}$.


\paragraph{Associational and causal estimands.} In causal inference, we are mostly interested in establishing the following models \citep{pearl2009,hernan_robins_whatif}:
\begin{align}
\quad m_t(x) &:= \E[Y \mid T=t, X=x] \label{eq:obs_mod}\\
\quad g_t(x) &:= \E[Y \mid \mathrm{do}(T=t), X=x] \label{eq:caus_mod}
\end{align}
\normalsize
where $m_t(x)$ refers to the expected values of $Y$ among all subjects with $(X,T)=(x,t)$, and $g_t(x)$ represents a hypothetical average value of $Y$ among all subjects with $X=x$ intervened on treatment $t$ (no matter what treatments they actually receive); we refer to $m_t$ and $g_t$ as associational and causal outcome models, respectively. 
The associational expectation \eqref{eq:obs_mod} cannot be used for decision making because, for instance, if a certain pattern $x$ is always associated with treatment $t$, then the model never observes outcomes under any alternative $t'$ for those $x$. Any prediction at $X=x$ and $T=t'$ is pure extrapolation driven only by parametric/inductive assumptions, not data. Hence plug-in decisions based on those counterfactual predictions can be arbitrarily wrong.
The causal expectation model \eqref{eq:caus_mod}, on the other hand, can be used to facilitate decision-making, but it is difficult to estimate, especially without extra information like instrumental variables. 
Therefore, our optimal strategy here is to consider the scenarios where the two models coincide.

\paragraph{Potential outcome framework.} 
We adopt the Neyman–Rubin framework: for $t\in\mathcal{T}$, let $Y(t)$ denote the potential outcome had treatment $t$ been assigned. We assume the standard identification conditions: (i) \emph{Consistency/SUTVA} (the observed outcome under the received treatment equals the corresponding potential outcome, with well-defined interventions); (ii) \emph{Positivity/overlap} ($0<p(T=t\mid X=x)$ on the support of $(X,t)$); and (iii) \emph{Conditional ignorability} ($Y(t)\!\perp\!\!\!\perp\! T \mid X$). 
Under the consistency and ignorability assumptions, the following identity holds \citep{robins1986,pearl2009,hernan_robins_whatif}:
\begin{equation}
\begin{aligned}
\E[Y(t)|X=x] &:= \E[Y|\mathrm{do}(T=t),X=x]\\
&= \E[Y|T=t,X=x].
\end{aligned}\label{eq:ign_ass}
\end{equation}
The \textit{structural causal model (SCM)} in Figure \ref{fig:scms} (a) aligns with the settings in most clinical research as the covariates $X$ causally affect both the treatment $T$ and the outcome $Y$ (e.g. patients who receive transfusions are often clinically sicker). 
If we further assume that $X$ and $T$ are independent, i.e., $p(x|t) = p(x)$, we obtain the setting of a \textit{randomized controlled trial (RCT)}, as illustrated in Figure~\ref{fig:scms} (b). 
RCTs are regarded as the gold standard for causal inference, but they can be very costly to conduct and, in some cases, ethically problematic. 
In the remainder of this section, we discuss how to identify and adjust for the bias introduced by the causal dependence between $X$ and $T$.

\begin{figure}[!tbp]
\begin{center}
\begin{tabular}{ccc}
\includegraphics[clip, width=.13\textwidth]{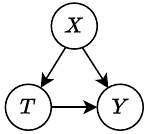} &
\includegraphics[clip, width=.14\textwidth]{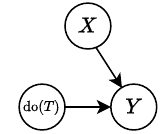}&
\includegraphics[clip, width=.13\textwidth]{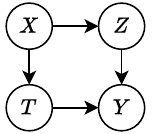} \\
(a) & (b) & (c) 
\end{tabular}
\end{center}
\vspace{-1.5em}
\caption{Structural causal models: (a) an SCM with covariates satisfying the ignorability assumption, 
(b) randomized controlled trial (RCT), 
(c) $Z$ as a learned representation of $X$.}
\label{fig:scms}
\end{figure}

\paragraph{Identifying the bias.} 
Assume consistency, positivity, and ignorability hold, and let $T\sim\pi$ denote the treatment with law $\pi$. Write $p_t(x):=p(x|T=t)$ for the covariate law conditional on $T=t$, and $p_x(x):=\int p_t(x)\pi(dt)$ for the marginal of $X$. 
Fix a predictor $g_t:\mathcal{X}\to\mathcal{Y}$ for arm $t$ and define the (deterministic) loss profile as 
\begin{equation}
\begin{aligned}
\phi_t(x)\;&:=\;\E\!\big[L\big(Y(t),g_t(x)\big)\,\big|\,X=x\big] \\
&= \int L(y,g_t(x)) p(y|x,\mathrm{do}(T=t)) dy.
\end{aligned}\label{eq:loss_prof}
\end{equation}
We define the factual and cross-arm risks for arm $t$ as
\begin{align}
R_t^F := \E_{p(x|t)}[\phi_t(x)],\quad R_{t\rightarrow t'}^{CF} := \E_{p(x|t')}[\phi_t(x)].
\end{align}
We also define the counterfactual (marginal) risk as 
\begin{align}
R_t^{CF} := \E_{\pi(t')}[R_{t\rightarrow t'}^{CF}] = \E_{p_x(x)}[\phi_t(x)],
\end{align}
and the population aggregates $R^F:=\E_{t\sim\pi}[R_t^F]$ and $R^{CF}:=\E_{t\sim\pi}[R_t^{CF}]$. 
Let $\mathcal F$ be a function class and suppose there exists $\lambda>0$ such that
$\phi_t/\lambda\in\mathrm{conv}(\mathcal F)$ for all $t\in\mathcal T$.
Then, for any $t,t' \in \mathcal{T}$,
\begin{equation}
\begin{aligned}
\big|R_{t\to t'}^{CF}-R_t^F\big|
&=\big|\E_{p_{t'}}[\phi_t(x)]-\E_{p_t}[\phi_t(x)]\big|\\
&\le\;\lambda\,\mathrm{IPM}_{\mathcal F}(p_{t'},p_t),
\end{aligned}
\end{equation}
where $\mathrm{IPM}_{\mathcal F}(p,q):=\sup_{f\in\mathcal F}\big|\E_p [f(x)]-\E_q [f(x)]\big|$.
Averaging over $t'\sim\pi$ yields
\begin{align}
\big|R_t^{CF}-R_t^F\big|
\le \lambda\int \mathrm{IPM}_{\mathcal F}(p_{t'},p_t)\,\pi(dt').    
\end{align}
A further average over $t\sim\pi$ gives the population bound \citep{kazemi2024adversarially}
\begin{align}
R^{CF} \le R^F + \lambda\iint \mathrm{IPM}_{\mathcal F}(p_t,p_{t'})\pi(dt)\pi(dt').\label{eq:cf_1}
\end{align}
In particular, if the conditional covariate laws $\{p_t\}_{t\in\mathcal T}$ are close in the $\mathrm{IPM}_{\mathcal F}$ sense, the counterfactual risk cannot exceed the factual risk by much.

\paragraph{Relations with information theory.}
Given the function class $\mathcal{F} = \{f:\|f\|_\infty \leq 1\}$, we have $\mathrm{IPM}_\mathcal{F}(p_t, p_{t'}) = 2\mathrm{TV}(p_t, p_{t'})$. Based on \eqref{eq:cf_1}, we have 
\small
\begin{equation}
\begin{aligned}
R^{CF} &\leq R^F + 2\lambda \iint \mathrm{TV}(p_t, p_{t'}) \pi(dt)\,\pi(dt') \\
&\overset{(a)}{\leq} R^F + 4\lambda \int \mathrm{TV}(p_t, p_x) \pi(dt)\\
&\overset{(b)}{\leq} R^F + 2\sqrt{2}\lambda \int \sqrt{\KL(p_t\| p_x)} \pi(dt)\\
&\overset{(c)}{\leq}  R^F + 2\sqrt{2}\lambda \sqrt{\E_{\pi(t)}[\KL(p(x|t)\|p(x))]}\\ 
&= R^F + 2\sqrt{2}\lambda \sqrt{ I(X;T)},
\end{aligned}\label{eq:cf_it}
\end{equation}
\normalsize
where inequalities (a), (b), and (c) are due to triangular inequality, Pinsker's inequality, and Jensen's inequality. Thus, reducing the mutual information $I(X;T)$ tightens the counterfactual–factual risk gap.

\paragraph{Related works.}
The bound established in \eqref{eq:cf_it} motivates the learning of representations $Z=\phi(X)$ that support factual prediction and bring the treated/control representation distributions closer. 
In the static setting, balanced-representation methods such as CFR/TARNet explicitly penalize a discrepancy (e.g., disc$_{\mathcal H}$ or MMD) between $p(Z\mid T{=}1)$ and $p(Z\mid T{=}0)$ while fitting outcomes, thereby shrinking the bound \eqref{eq:cf_it} \citep{johansson2016,shalit2017}. Beyond these, we highlight two complementary strategies that fit naturally within our bound:
\begin{enumerate}[label=(\alph*)]
\item \textbf{Information-regularized representations.} Learn features $Z=\phi(X)$ that are predictive for $Y$ yet uninformative about $T$ by solving
\small
\begin{equation}
\begin{aligned}
&\min_{\phi}\ \E\!\big[L\big(y, g(\phi(x),t)\big)\big] + \lambda\, I\!\big(T;\phi(X)\big)\\
\propto\,&\min_{\phi}\ \max_{\theta}\ \E\big[L\big(y, g(\phi(x),t)\big)\big] + \lambda \E\!\big[\log p_{\theta}\!(t|\phi(x))\big],
\end{aligned}
\end{equation}
\normalsize
where a variational classifier $p_\theta(t\!\mid\!\phi(x))$ provides a tractable surrogate for $I(T;Z)$. 
\item \textbf{Treatment-invariant joint encodings.} Learn a joint feature $Z=\phi(X,T)$ and enforce invariance across (possibly clustered) treatments by adversarially discouraging a discriminator from predicting the treatment cluster $c$ \citep{berrevoets2020organite}:
\begin{equation}
\begin{aligned}
&\min_{\phi}\ \E\!\big[L\big(y, (g\!\circ\!\phi)(x,t)\big)\big] + \lambda\, I\!\big(C;\phi(X,T)\big)\\
\propto\,&\min_{\phi}\ \max_{\theta}\ \E\!\big[L\big(y, (g\!\circ\!\phi)(x,t)\big)\big] \\
&\qquad \qquad \qquad + \lambda\, \E\!\big[\log p_{\theta}\!\big(c\mid \phi(x,t)\big)\big].
\end{aligned}
\end{equation}
\normalsize
Pushing $p(\phi(X,T)\!\mid\!C{=}c)$ together across clusters contracts the discrepancy between treated/control representations and likewise tightens the bound. 
\end{enumerate}
Orthogonal lines include reweighting and orthogonalization—IPTW/stabilized weights, entropy balancing/kernel mean matching, and double/debiased ML—which address assignment bias via importance weighting or Neyman-orthogonal scores rather than representation constraints \citep{rosenbaum1983central,robins2000marginal,hainmueller2012entropy,huang2007correcting,chernozhukov2018double}. In the \emph{dynamic setting} with time-dependent confounding, representation learning is extended along trajectories (e.g., CRN, Causal Transformer) to maintain balance sequentially \citep{bica2020crn,melnychuk2022ct}.


\section{SICE FORMULATION}
\label{sec:sice}
Let $Z$ be a representation of $X$ via a stochastic map $z \sim q_\phi(z|x)$; thus $T \leftarrow X \rightarrow Z$ is a Markov structure.
For any treatment $t$, the induced conditional law of $Z$ is $p(z|t) = \int p(x|t)q_\phi(z|x) dx$, and if $T\sim \pi$ we write the $T$-marginalized representation distribution as $p_z(z) = \int p(z|t)\pi(dt)$. We redefine the loss profile in \eqref{eq:loss_prof} as 
\begin{align}
\phi_t(z) := \E[L(Y(t), g_t(z))|Z=z].
\end{align}
and the risks as 
\begin{align}
R^F = \E_{\pi(t)p(z|t)}[\phi_t(z)],\quad R^{CF} = \E_{\pi(t)p_z(z)}[\phi_t(z)].
\end{align}
By adapting the argument in Sec. \ref{sec:bg_causal_ml} to the representation space and comparing $p(z|t)$ with the marginal $p_z(z)$, we obtain
\begin{align}
R^{CF} \leq R^F + \sqrt{2}\lambda \sqrt{I(Z;T)} \approx \hat R^F+ \sqrt{2}\lambda \sqrt{I(Z;T)},\label{eq:cf_it_2}
\end{align}
where due to consistency
\begin{align}
\hat R^F = \E_{p_\cD(x,y,t)q_\phi(z|x)}[L(y,g_t(z))].
\end{align}
By minimizing the bound \eqref{eq:cf_it_2}, we compress away $T$-predictive idiosyncrasies while retaining outcome-relevant signal via the decoder term, encouraging $Z$ to retain outcome-relevant information while reducing treatment-revealing information.
The proposed model is illustrated in Figure \ref{fig:scms} (c) with a Markov chain $X \rightarrow Z \rightarrow Y$. 

\paragraph{Assumptions for SICE.} We retain Consistency/SUTVA. 
Overlap is required on the representation: there exists $\delta\in(0,1)$ such that for all $(z,t)$ with $p_Z(z)>0$, $\delta \leq p(T=t|Z=z) \leq 1-\delta$. The ignorability assumption is updated as $Y(t)\indep T|Z$, for all $t\in \cT$. 
We assume $Z = \phi(X)$ is outcome-sufficient and confounding-sufficient so that $Y(t)\indep T|Z$. We enforce approximate treatment-agnosticity via the MI regularizer to reduce assignment bias, while preserving outcome signal via the decoder loss. 

\paragraph{Adjusting for the bias.}
Motivated by Eq.~\eqref{eq:cf_it}, we seek a representation that is maximally predictive for $Y$ while minimally informative about $T$. 
Let $z\sim q_\phi(z\mid x)$ be a stochastic encoder and write expectations over $(x,y,t)\sim p_\mathcal{D}$ and $z\sim q_\phi(\cdot|x)$. We optimize
\begin{align}
\min \quad \E[L(y, g(z,t))] + \lambda I(t; z)\label{eq:sice_info}
\end{align}
Following the approach of \citep{no_dat}, we decompose the mutual information as
\begin{align}
I(z; t) = I(z; t|x) + I(z;x) - I(z;x|t),\label{eq:it1}
\end{align}
where 
$$I(z; t|x) = H(z|x) - H(z|x, t) = 0,$$
where $H(z|x,t) = H(z|x)$ because $p(z|x, t) = p(z|x)$; this can also be inferred by noting the conditional independence between $z$ and $t$ given $x$. We now focus on lower-bounding $I(z;x|t)$,
\begin{equation}
\begin{aligned}
I(z;x|t) &= H(x|t) - H(x|z,t)\\
&= H(x|t) + \E[\log p(x|z,t)]\\
&\geq \E[\log p_\psi(x|z,t)] + \text{Const.} 
\end{aligned}\label{eq:sice_vb_bound}
\end{equation}
Similarly to variational information bottleneck (VIB) \citep{vib}, we can upper-bound $I(x;z)$ as
\begin{align}
I(x;z) \leq \E[\log q_\phi(z|x) - \log r(z)].\label{eq:bound}
\end{align}
Therefore, we have
\begin{align}
I(z;t) \leq \E[\log q_\phi(z|x) - \log r(z)- \log p_\psi(x|z,t)]+c.  
\end{align}
Putting it all together, we have the formulation, 
\begin{equation}
\begin{aligned}
\min_{g,\phi,\psi} \quad &\E_{p_\cD(x,y,t)q_\phi(z|x)}\big[L(y, g(z,t)) - \lambda \log p_\psi(x|z,t)\big] \\
&\qquad \qquad + \lambda \E_{p_\cD(x)}[\KL(q_\phi(z|x)\|r(z))] 
\end{aligned}
\label{eq:sice}
\end{equation}
The learning problem is cast as a single objective with a scalar information penalty, rather than a min–max game against a domain classifier. 
Moreover, thanks to the tractable variational bound in \eqref{eq:sice_vb_bound}, SICE naturally accommodates complex treatment spaces—including high-dimensional or continuous $t$ (as is verified in Section \ref{sec:exp})—where enumerating or clustering “domains” becomes impractical and standard DAT methods struggle to scale.

\paragraph{Intervention.} To estimate the individualized treatment effect at a given covariate $x$, we compute the posterior over $z$ given $x$, then plug the treatment intervention into the outcome head while holding $z$ fixed. Mathematically, we have
\begin{equation}
\begin{aligned}
\mathrm{ITE}(x;t,t') &= \E[Y(t) - Y(t')|X=x] \\
&= \E_{q_\phi(z|x)}[g(z,t) - g(z,t')].
\end{aligned}
\end{equation}
This is a valid ``do''-intervention under the representation-sufficiency assumption $Y(t)\indep T|Z$; under this assumption, intervening on $T$ changes the outcome head while leaving the posterior over $z$ given $x$ unchanged.


\section{DICE FORMULATION}
\label{sec:dice}

\paragraph{Notations.}
Suppose there are $T$ decision points. 
We use $\bfv \in \mathcal{V}$ to denote the time-invariant covariates, which could contain patient demographics and preoperative assessments. 
Let $\bfx_t \in \mathcal{X}$ denote the time-varying covariates at decision point $t$, such as patients' vital signs. 
We use $\bfa_t \in \mathcal{A}$ to denote the action at time $t$. 
We use $\bfy_{t+1}$ to denote the response at each time point, e.g. the tumor size. 
We define the history up to time $t$ as $h_t = \{\bfv, \bfx_{1:t}, \bfa_{1:t-1}\}$. In practice, we can use RNN hidden states to encode the historical information. A graphical representation of DICE is shown in Figure \ref{fig:dice}. 
In this Section, we use the term ``RNN'' to encompass a spectrum of recurrent architectures such as vanilla RNNs, gated recurrent units (GRUs) \citep{cho2014gru}, and long short-term memory networks (LSTMs)\citep{hochreiter1997lstm}, and the RNN backbone can be easily replaced by a transformer and other structures of the same purpose \citep{causal_transformer}. 

\begin{figure}[!tbp]
\begin{center}
\begin{tabular}{cc}
\includegraphics[clip, trim=.9cm 0cm .7cm 0cm, width=.27\textwidth]{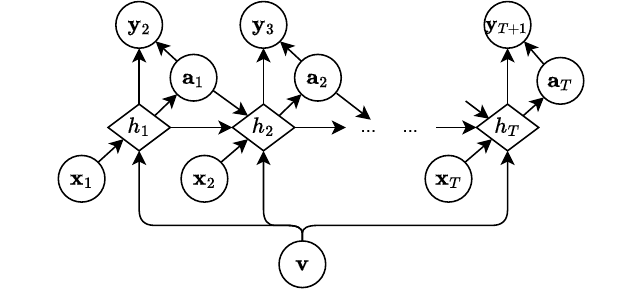}& \includegraphics[clip, trim=.3cm 0cm .3cm 0cm, width=.18\textwidth]{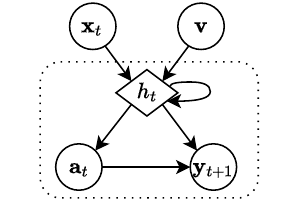}  \\
(a) & (b)
\end{tabular}
\end{center}
\vspace{-1.5em}
\caption{DICE: (a) Graphical representation of DICE, where circular nodes denote observed variables and diamond nodes represent RNN hidden states; (b) simplified illustration of its single recurrent segment.}
\label{fig:dice}
\end{figure}

\paragraph{Assumptions.}
For each $t=1,\dots,T$, we assume (i) \emph{consistency}: if the realized treatment history equals $\bar \bfa_{1:t}$, then $Y_{t+1}=Y_{t+1}(\bar \bfa_{1:t})$;
(ii) \emph{sequential ignorability}: conditional on feature $z_t$, the current treatment is independent from the potential outcome, $Y_{t+1}(a_t)\indep a_t|z_t$; and (iii) \emph{sequential positivity}: for any $h$ with $p(h_t=h)>0$, we have $p(\bfa_t=a| h_t=h)>0$ for every feasible action $a$ under evaluation.

\paragraph{Adjusting for the bias.} 
As shown in Figure \ref{fig:dice} (b), the recurrent segment of DICE is comparable to the SCM of its static counterpart shown in Figure \ref{fig:scms} (a), with $h_t$ taking the role of $x$ in the static case.
The treatment-assignment bias in the dynamic model mostly comes from the association between the history $h_t$ containing time-varying confounders $\bfx_t$ and the current treatment $\bfa_t$. 
To remove such bias, we can learn a stochastic representation $z_t \sim q_\phi(z_t|h_t)$ that is treatment-agnostic and predictive on $y_{t+1}$.
We adopt the following parameterization: 
\small
\begin{equation}
\begin{aligned}
&h_t = d_h(h_{t-1}, \bfv, \bfa_{t-1},\bfx_t), \, \bfz_t \sim q_\phi(\bfz_t|h_t), \, \hat\bfy_{t+1} = g(\bfz_t,\bfa_t),
\end{aligned}\label{eq:para_rnn_1}
\end{equation}
\normalsize
for $t\in [T]$, where $d_h$ and $g$ are neural network predictors, and $q_\phi$ is stochastic embedding function. We establish the objective:
\begin{align}
\min\quad \sum_{t=1}^T \E[L(\bfy_{t+1}, g(\bfz_t, \bfa_t))] + \lambda \sum_{t=1}^T I(\bfz_t;\bfa_t).
\label{eq:dice_obj}
\end{align}
The training objective is given by
\begin{equation}
\begin{aligned}
\min_{d_h, \phi, g, \psi}\quad &\sum_{t=1}^T \E_{p_\cD(\bfx_{t}, \bfa_t, \bfy_{t+1})q_\phi(\bfz_t|h_t)}[L(\bfy_{t+1}, g(\bfz_t, \bfa_t))\\
&\qquad \qquad - \lambda \log p_\psi(h_t|\bfz_t, \bfa_t)]\\
&\qquad \qquad + \lambda \E_{p_\cD(h_t)}[\KL(q_\phi(\bfz_t|h_t)\|r(\bfz_t))].
\end{aligned}
\label{eq:dice_train_obj}
\end{equation}

\begin{table*}[!h]
\centering
\caption{Comparison across treatment dimensions $d_t$ (lower is better). Best is \textbf{bold} $+\star$; second-best is \underline{underlined} $+\square$.}
\label{tab:big_grid}
\setlength{\tabcolsep}{5pt}
\renewcommand{\arraystretch}{1.08}

{\bf (a) ATE Error}\par\vspace{0.2em}
\resizebox{0.94\textwidth}{!}{%
\begin{tabular}{lccccccc}
\toprule
Method & $2$ & $5$ & $10$ & $12$ & $14$ & $16$ & $18$ \\
\midrule
SICE           & \textbf{0.0006}$^{\star}$ & \underline{0.0019}$^{\square}$ & \textbf{0.0001}$^{\star}$ & 0.0141 & \textbf{0.0005}$^{\star}$ & \textbf{0.0012}$^{\star}$ & \textbf{0.0069}$^{\star}$ \\
TARNet         & \underline{0.0054}$^{\square}$ & 0.0022 & 0.0089 & \underline{0.0006}$^{\square}$ & 0.0133 & 0.0210 & 0.0406 \\
CFR-HSIC       & \underline{0.0054}$^{\square}$ & 0.0081 & 0.0068 & 0.0133 & 0.0326 & 0.0131 & 0.0211 \\
DragonNet-GPS  & 0.0099 & 0.0076 & 0.0196 & \textbf{0.0001}$^{\star}$ & 0.0128 & 0.0219 & 0.0179 \\
GRL            & 0.0099 & 0.0043 & \underline{0.0013}$^{\square}$ & 0.0169 & 0.0116 & 0.0149 & \underline{0.0131}$^{\square}$ \\
CFR-Wass       & 0.0086 & \textbf{0.0008}$^{\star}$ & 0.0213 & 0.0080 & \underline{0.0029}$^{\square}$ & \underline{0.0088}$^{\square}$ & 0.0187 \\
\bottomrule
\end{tabular}
}

\vspace{0.6em}

{\bf (b) $\mathrm{RMSE}_y$}\par\vspace{0.2em}
\resizebox{0.94\textwidth}{!}{%
\begin{tabular}{lccccccc}
\toprule
Method & $2$ & $5$ & $10$ & $12$ & $14$ & $16$ & $18$ \\
\midrule
SICE           & \textbf{0.5375}$^{\star}$ & \underline{0.6241}$^{\square}$ & \underline{0.8261}$^{\square}$ & \textbf{0.6534}$^{\star}$ & 0.6855 & \textbf{0.8022}$^{\star}$ & \textbf{0.7361}$^{\star}$ \\
TARNet         & 0.5901 & 0.6297 & 0.8818 & 0.6994 & \textbf{0.6635}$^{\star}$ & 0.8196 & 0.7682 \\
CFR-HSIC       & 0.6047 & 0.6999 & 0.9412 & 0.8373 & 0.7452 & 0.9896 & 0.8906 \\
DragonNet-GPS  & \underline{0.5565}$^{\square}$ & \textbf{0.6215}$^{\star}$ & 0.8658 & \underline{0.6886}$^{\square}$ & 0.6957 & \underline{0.8037}$^{\square}$ & 0.7617 \\
GRL            & 0.5874 & 0.6288 & \textbf{0.7895}$^{\star}$ & 0.6978 & \underline{0.6698}$^{\square}$ & 0.9102 & \underline{0.7396}$^{\square}$ \\
CFR-Wass       & 0.6064 & 0.7119 & 0.9097 & 0.7630 & 0.6916 & 1.0837 & 0.9585 \\
\bottomrule
\end{tabular}
}

\vspace{0.6em}

{\bf (c) PEHE}\par\vspace{0.2em}
\resizebox{0.8\textwidth}{!}{%
\begin{tabular}{lccccccc}
\toprule
Method & $2$ & $5$ & $10$ & $12$ & $14$ & $16$ & $18$ \\
\midrule
SICE           & \underline{0.1635}$^{\square}$ & \textbf{0.2385}$^{\star}$ & 0.5480 & \textbf{0.3769}$^{\star}$ & \underline{0.5042}$^{\square}$ & \textbf{0.5387}$^{\star}$ & \textbf{0.4675}$^{\star}$ \\
TARNet         & \textbf{0.1610}$^{\star}$ & \underline{0.2498}$^{\square}$ & \textbf{0.4811}$^{\star}$ & \underline{0.4112}$^{\square}$ & \textbf{0.4983}$^{\star}$ & 0.5942 & 0.4896 \\
CFR-HSIC       & 0.1738 & 0.3553 & 0.6717 & 0.6039 & 0.5888 & 0.8894 & 0.7399 \\
DragonNet-GPS  & 0.1774 & 0.2682 & 0.4983 & 0.4144 & 0.5178 & \underline{0.5467}$^{\square}$ & \underline{0.4737}$^{\square}$ \\
GRL            & 0.1872 & 0.2577 & \underline{0.4843}$^{\square}$ & 0.4456 & 0.5080 & 0.6132 & 0.4830 \\
CFR-Wass       & 0.1886 & 0.3333 & 0.5285 & 0.4946 & 0.5382 & 0.8298 & 0.8059 \\
\bottomrule
\end{tabular}
}

\end{table*}

\section{EXPERIMENTAL VALIDATION}
\label{sec:exp}
The code for replicating the experiments in this section can be found at \url{https://github.com/chrisfngr/AFCP}.

\subsection{SICE: Static Model}
We evaluate the static counterfactual prediction model \textsc{SICE} against strong baselines across a broad spectrum of treatment dimensionalities, including \emph{ultra-high-dimensional} regimes. To stress-test scalability and realism, we vary $d_t$ from small/medium scales to hundreds of dimensions—settings commonly encountered in personalized medicine (multi-drug or multi-exposure regimens) and recommendation/targeted advertising (large action catalogs and multi-action bundles). On synthetic benchmarks where counterfactuals are available, we report \textbf{ATE Error}, $\mathbf{RMSE}_y$, and \textbf{PEHE} (all lower is better). Baselines include TARNet/CFR with IPM regularization (incl.\ Wasserstein) \cite{shalit2017}, a CFR variant with HSIC independence regularization \cite{Gretton2008HSIC}, DragonNet augmented with a generalized propensity-score head \cite{Shi2019Dragonnet,HiranoImbens2004Biometrika}, and an adversarial learner with a gradient reversal layer (GRL) \cite{ganin2016domain}.All baselines follow the same preprocessing and training budget as \textsc{SICE}.

\subsubsection{Synthetic Data}
Unless otherwise noted, all models are trained for \textbf{50 epochs} using the \textbf{Adam optimizer} (learning rate $5\times10^{-4}$) under a unified preprocessing pipeline and training budget. Unless specified, experiments run on a single \textbf{NVIDIA A100 (80GB)} GPU. Baseline models employ $128$-dimensional representation layers. Across experiments, we report three metrics: \textbf{ATE Error}, $\mathbf{RMSE}_y$, and \textbf{PEHE} (all lower is better).


\begin{figure*}[!t]
\begin{center}
\scalebox{0.98}{
\begin{tabular}{@{}ccc@{}}
\includegraphics[clip,trim=.3cm 0cm .3cm 0cm,width=.31\linewidth]{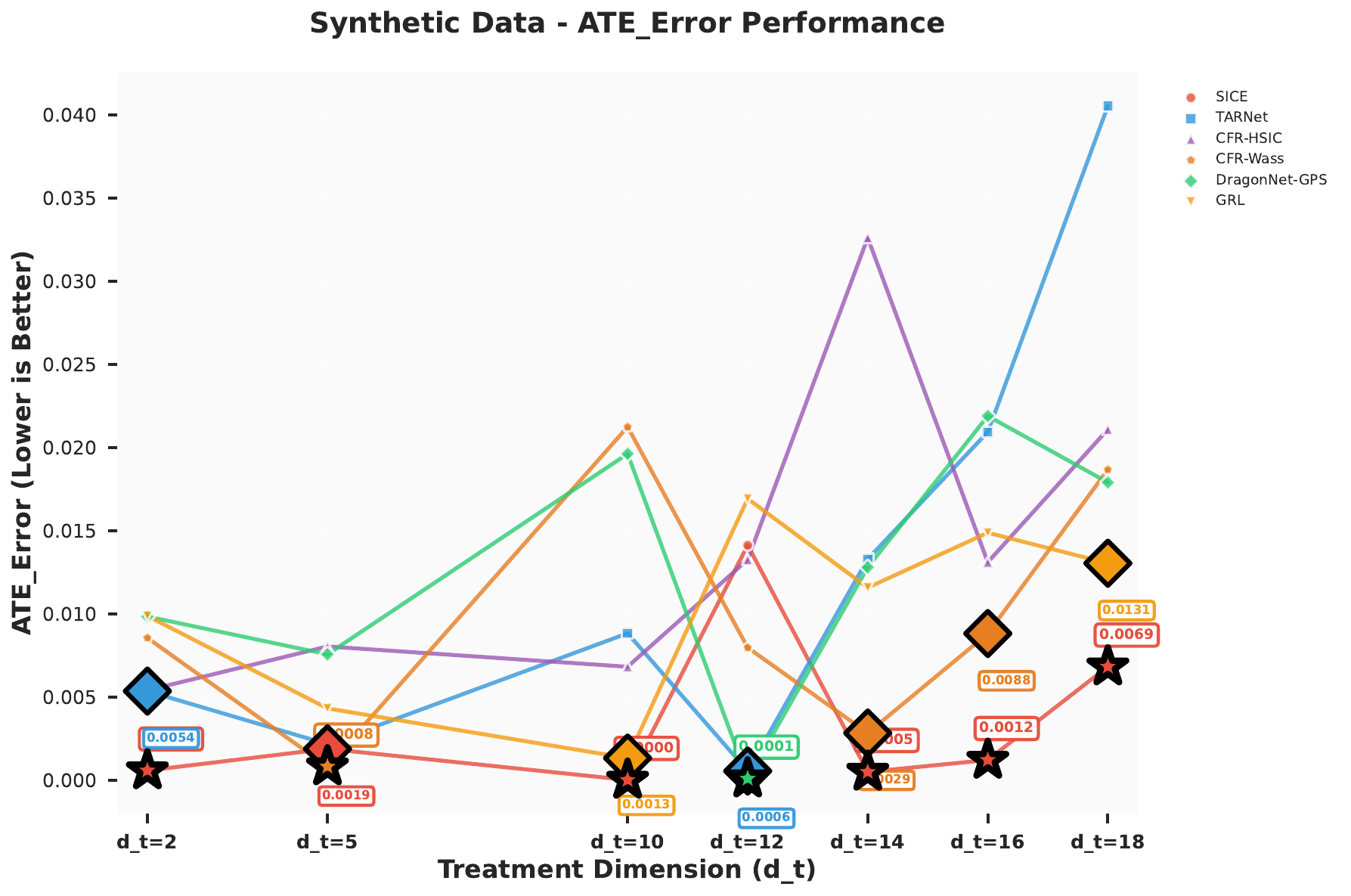}& 
\includegraphics[clip,trim=.3cm 0cm .3cm 0cm,width=.31\linewidth]{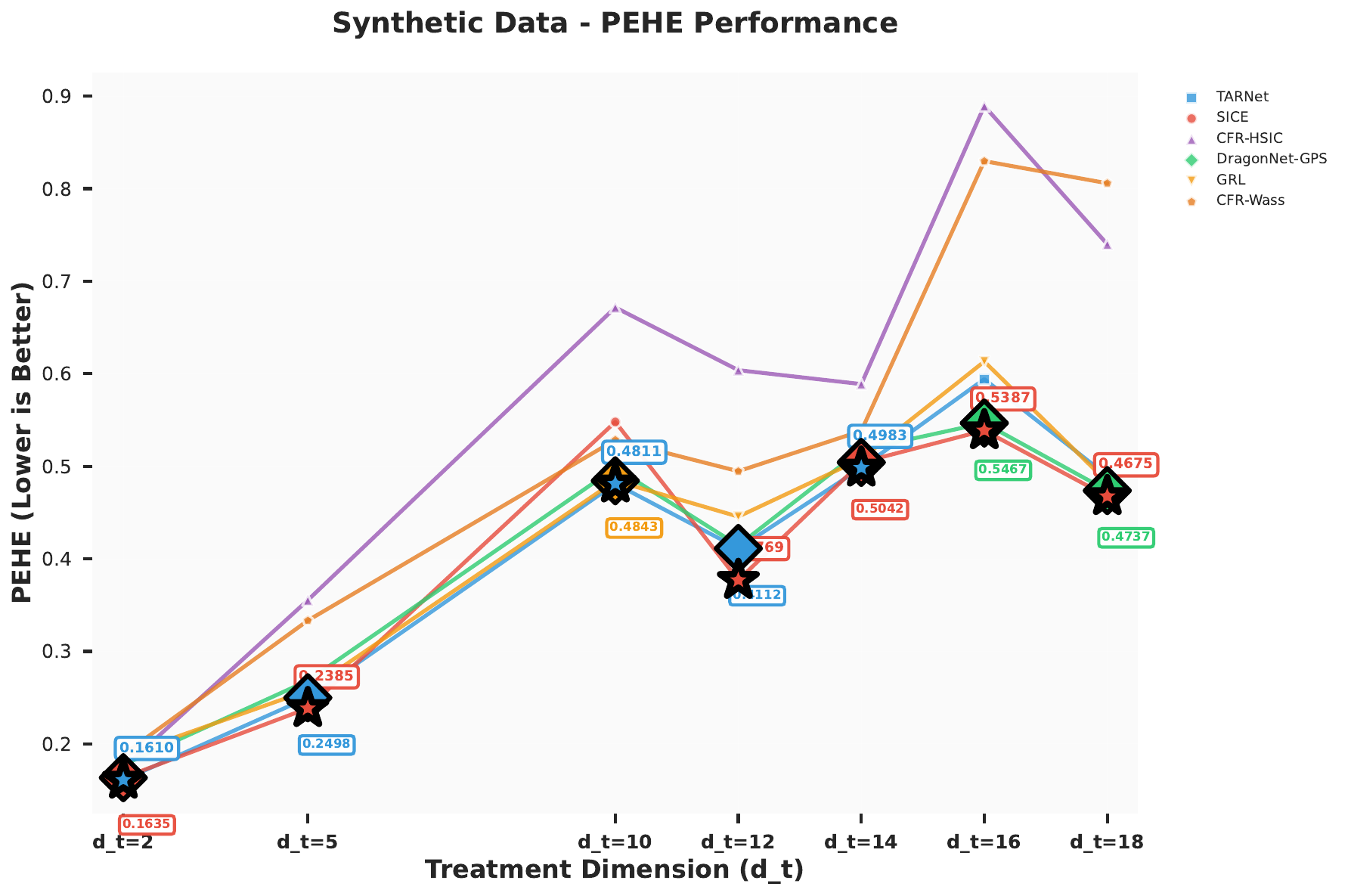}&
\includegraphics[clip,trim=.3cm 0cm .3cm 0cm,width=.31\linewidth]{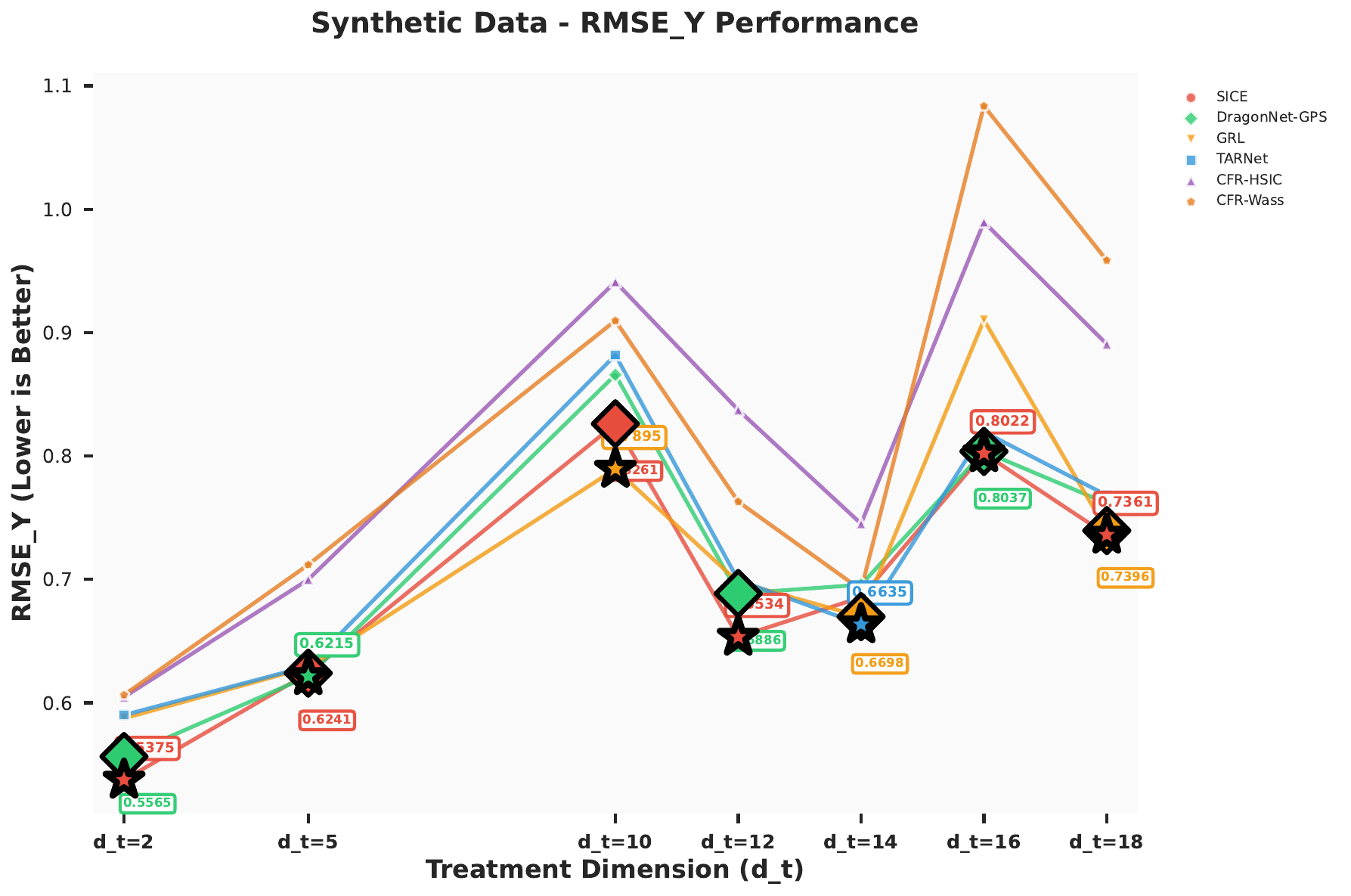}\\
(a) ATE Error & (b) PEHE  & (c) 
$\mathrm{RMSE}_y$ \\
\end{tabular}
}
\end{center}
\setlength{\abovecaptionskip}{0pt}  
\caption{Comparison across treatment dimensions $d_t$. Panels (a)–(c) compare \textsc{SICE} with other methods on ATE Error, PEHE, and $\mathrm{RMSE}_y$ (all lower is better). Star markers denote the best method; square markers denote the second-best.}
\label{fig:cluster}
\end{figure*}

\begin{table}[!t]
\centering
\caption{Comparison of PEHE on the Synthetic Dataset across different treatment dimensions $d_t$ (lower is better).}
\label{tab:high_dim_treatment_nonlinear_pehe}
\setlength{\tabcolsep}{5pt}
\renewcommand{\arraystretch}{1.08}

\textbf{Synthetic Dataset}

\begin{tabular}{lccc}
\toprule
Method & $d_t=200$ & $d_t=500$ & $d_t=1000$ \\
\midrule
SICE            & \textbf{1.18}$^{\star}$ & \textbf{3.44}$^{\star}$ & \textbf{8.34}$^{\star}$ \\
CFR--HSIC       & \underline{1.25}$^{\square}$ & \underline{3.53}$^{\square}$ & \underline{8.52}$^{\square}$ \\
TARNet          & 1.30 & 3.65 & 8.59 \\
GRL             & 1.36 & 3.67 & 8.65 \\
CFR--Wass       & 1.33 & 3.71 & 8.63 \\
DragonNet--GPS  & 1.31 & 3.69 & 8.70 \\
\bottomrule
\end{tabular}
\end{table}

\begin{table*}[!t]
\centering
\caption{NHANES 2017--2018 with multi-dimensional treatments ($T\!\in\!\{0,1\}^{82}$) and covariates ($X\!\in\!\mathbb{R}^{14}$).
We report factual prediction (RMSE/MAE; lower is better) and off-policy targeting (AUUC; higher is better).
Best is \textbf{bold}$+\star$, second-best is \underline{underlined}$+\square$.}
\label{tab:nhanes_multimetrics}
\setlength{\tabcolsep}{6pt}
\renewcommand{\arraystretch}{1.08}

\begin{minipage}[t]{0.48\textwidth}
\centering
\textbf{(a) Factual prediction: RMSE / MAE}\\[0.25em]
\begin{tabular}{lcc}
\toprule
\textbf{Method} & \textbf{RMSE} $\downarrow$ & \textbf{MAE} $\downarrow$ \\
\midrule
SICE (ours)      & \textbf{0.0285}$^{\star}$     & \textbf{0.0227}$^{\star}$ \\
TARNet           & 0.0938                        & 0.0823 \\
CFR--HSIC        & \underline{0.0899}$^{\square}$ & \underline{0.0817}$^{\square}$ \\
DragonNet--GPS   & 0.1332                        & 0.1225 \\
GRL              & 0.1001                        & 0.0902 \\
CFR--Wass        & 0.1008                        & 0.0912 \\
\bottomrule
\end{tabular}
\end{minipage}\hfill
\begin{minipage}[t]{0.48\textwidth}
\centering
\textbf{(b) Off-policy targeting: AUUC}\\[0.25em]
\begin{tabular}{lc}
\toprule
\textbf{Method} & \textbf{AUUC} $\uparrow$ \\
\midrule
SICE (ours)      & \underline{0.2761}$^{\square}$ \\
TARNet           & 0.2383 \\
CFR--HSIC        & 0.2393 \\
DragonNet--GPS   & \textbf{0.3021}$^{\star}$ \\
GRL              & 0.2034 \\
CFR--Wass        & 0.2150 \\
\bottomrule
\end{tabular}
\end{minipage}

\vspace{0.35em}
\parbox{0.96\textwidth}{\footnotesize
\textit{Notes.} NHANES lacks ground-truth counterfactuals, so PEHE is not applicable. 
Arrows indicate the preferred direction for each metric. 
Best/second-best are evaluated across listed methods within each panel.
}

\end{table*}

Across three metrics and seven treatment dimensions, \textsc{SICE} attains \textbf{13/21} best and \textbf{5/21} second-best results (three metrics $\times$ $d_t\in\{2,5,10,12,14,16,18\}$).
Focusing on the higher-dimensional regime $d_t{>}2$ ($d_t\in\{5,10,12,14,16,18\}$), \textsc{SICE} is best in \textbf{11/18} and second-best in \textbf{4/18} metric–dimension pairs (only \textbf{3/18} cases outside the top two).
Breakdown for $d_t{>}2$: on \textbf{ATE Error} \textsc{SICE} is best in \textbf{4/6} (second at $d_t{=}5$); on $\mathbf{RMSE}_y$ it is best in \textbf{3/6} (second at $d_t{=}5,10$); on \textbf{PEHE} it is best in \textbf{4/6} (second at $d_t{=}14$).
The only clear non–top-two outliers are ATE at $d_t{=}12$ and PEHE at $d_t{=}10$; elsewhere \textsc{SICE} is either the best or runner-up.
Overall, performance degrades much more slowly than the baselines as $d_t$ increases and remains competitive through $d_t{=}18$, highlighting robustness in higher-dimensional treatments.

\paragraph{Discussion on High-Dimensional Treatment Experiment}
\label{sec:discussion_high_dim} 
We now evaluate \textsc{SICE} in the setting of \emph{ultra–high-dimensional treatments}, which frequently arise in real-world applications. Examples include personalized medicine (e.g., multi-drug or multi-exposure regimens) and recommendation/targeted advertising (e.g., large action catalogs or multi-action bundles), where the treatment can be a vector with hundreds of components. To stress-test scalability, we report supplementary results on a high-dimensional synthetic benchmark with $d_t\in\{200,500,1000\}$ (Table~\ref{tab:high_dim_treatment_nonlinear_pehe}). \textsc{SICE} attains the lowest PEHE across all three ultra-high-dimensional settings, indicating strong robustness as treatment dimensionality grows.

As shown in the table above, the results are clear: \textbf{SICE consistently outperforms all other baseline methods} across all tested treatment dimensions ($d_t = 200, 500, 1000$). For each dimension, SICE achieves the lowest PEHE (Precision in Estimating Heterogeneous Effects), a crucial metric for evaluating individual treatment effect estimation. This finding strongly supports our hypothesis that SICE is a superior and more robust choice for causal effect estimation in the presence of high-dimensional and non-linear data.

The ability of SICE to maintain its performance advantage as the treatment dimensionality scales up confirms its robustness and adaptability. This makes it a highly effective and reliable solution for causal inference tasks in today's complex, data-rich domains, where traditional methods often struggle.

\subsubsection{Real Data (NHANES 2017--2018)}
We use the \textbf{NHANES 2017--2018} U.S. health survey (CDC/NCHS) \cite{CDC_NHANES_2017_2018}. 
In our setup, the \emph{treatment} is a multi-dimensional vector of medication/exposure indicators with \textbf{$t_{\mathrm{dim}}=82$}; 
\emph{covariates} include core demographics/exam features with \textbf{$x_{\mathrm{dim}}=14$}; 
the \emph{outcome} $Y$ is continuous (e.g., a laboratory biomarker), so squared and absolute errors are meaningful for factual fit. 
Counterfactuals are unobserved; thus we pair factual prediction (RMSE/MAE) with an off-policy targeting metric (AUUC)~\cite{Rzepakowski2012KAIS} that evaluates ranking quality for beneficial interventions.


\paragraph{Interpretation.}
As summarized in Table~\ref{tab:nhanes_multimetrics}, \textbf{SICE} achieves \emph{state-of-the-art factual accuracy} under high-dimensional multi-treatment inputs ($t_{\mathrm{dim}}{=}82$, $x_{\mathrm{dim}}{=}14$), attaining the best RMSE and MAE. On the causal targeting axis, SICE ranks \emph{second} by AUUC, close to the uplift-specialized DragonNet--GPS. This profile---top factual calibration with competitive off-policy ranking---is particularly attractive for clinical decision support, where accurate outcome modeling is primary and reliable targeting quality is an added benefit in high-dimensional treatment settings.

\subsection{DICE: Dynamic Model}
We next assess the \textbf{Dynamic Information-theoretic Counterfactual Estimator (DICE)} in dynamic treatment settings. 
Synthetic data ($N=1000$ samples over $T=10$ time steps) is generated via a configurable \textbf{dynamic simulation} with $\mathbf{X}\in\mathbb{R}^8$, $\mathbf{Y}\in\mathbb{R}^1$, and treatment dimension $d_t$ varied up to $10$. 
All baselines follow the same setup as above. \textsc{DICE} employs a 128-dimensional hidden layer and learns a 32-dimensional latent representation $\mathbf{Z}$, controlled by $\lambda=10^{-5}$. 

\begin{table*}[!t]
\centering
\caption{DICE: Comparison across treatment dimensions $d_t$ (lower is better). Best is \textbf{bold}$+\star$; second-best is \underline{underlined}$+\square$.}
\label{tab:dice}
\setlength{\tabcolsep}{5pt}
\renewcommand{\arraystretch}{1.08}

\textbf{(a) ATE Error}

\begin{tabular}{lccccc}
\toprule
Method & $d_t=2$ & $d_t=4$ & $d_t=6$ & $d_t=8$ & $d_t=10$ \\
\midrule
DICE            & \textbf{0.0187}$^{\star}$ & \textbf{0.0029}$^{\star}$ & \textbf{0.0702}$^{\star}$ & \textbf{0.0951}$^{\star}$ & \textbf{0.0236}$^{\star}$ \\
TARNet          & \underline{0.0550}$^{\square}$ & \underline{0.2469}$^{\square}$ & \underline{0.2496}$^{\square}$ & \underline{0.1722}$^{\square}$ & \underline{1.6088}$^{\square}$ \\
CFR--HSIC       & 0.2363 & 0.1811 & 1.1234 & 2.1101 & 0.7033 \\
DragonNet--GPS  & 0.2381 & 0.2237 & 1.1421 & 2.2098 & 0.7269 \\
GRL             & 0.2385 & 0.4012 & 1.1048 & 2.0916 & 0.7451 \\
\bottomrule
\end{tabular}

\vspace{0.6em}

\textbf{(b) $\mathrm{RMSE}_y$}

\begin{tabular}{lccccc}
\toprule
Method & $d_t=2$ & $d_t=4$ & $d_t=6$ & $d_t=8$ & $d_t=10$ \\
\midrule
DICE            & \textbf{0.2288}$^{\star}$ & \textbf{0.2968}$^{\star}$ & \textbf{0.3538}$^{\star}$ & \textbf{0.3306}$^{\star}$ & \textbf{0.2215}$^{\star}$ \\
TARNet          & \underline{0.4801}$^{\square}$ & \underline{0.6326}$^{\square}$ & \underline{0.8753}$^{\square}$ & \underline{0.5056}$^{\square}$ & \underline{1.0363}$^{\square}$ \\
CFR--HSIC       & 1.8095 & 10.5588 & 7.5341 & 6.4323 & 2.4929 \\
DragonNet--GPS  & 1.8041 & 10.5683 & 7.5353 & 6.4346 & 2.4953 \\
GRL             & 1.8116 & 10.4921 & 7.5832 & 6.3553 & 2.4920 \\
\bottomrule
\end{tabular}

\vspace{0.6em}

\textbf{(c) PEHE}

\begin{tabular}{lccccc}
\toprule
Method & $d_t=2$ & $d_t=4$ & $d_t=6$ & $d_t=8$ & $d_t=10$ \\
\midrule
DICE            & \textbf{0.0293}$^{\star}$ & \textbf{0.1003}$^{\star}$ & \textbf{0.1180}$^{\star}$ & \textbf{0.1548}$^{\star}$ & \textbf{0.0823}$^{\star}$ \\
TARNet          & \underline{0.1161}$^{\square}$ & 0.2624 & \underline{0.3561}$^{\square}$ & \underline{0.3030}$^{\square}$ & 1.6686 \\
CFR--HSIC       & 0.2364 & \underline{0.2081}$^{\square}$ & 1.1264 & 2.1105 & \underline{0.7035}$^{\square}$ \\
DragonNet--GPS  & 0.2383 & 0.2384 & 1.1476 & 2.2105 & 0.7271 \\
GRL             & 0.2386 & 0.4320 & 1.1162 & 2.0928 & 0.7453 \\
\bottomrule
\end{tabular}

\end{table*}

Table~\ref{tab:dice} reports the outcomes. Across all $d_t$, \textsc{DICE} consistently achieves the \textbf{best scores} on ATE Error, $\mathbf{RMSE}_y$, and PEHE, clearly outperforming TARNet, CFR-HSIC, DragonNet-GPS, and GRL. 
Unlike baselines, which deteriorate sharply in higher treatment dimensions, \textsc{DICE} remains stable and accurate, demonstrating superior robustness for estimating ITEs in dynamic high-dimensional settings.




\section{CONCLUSION}
We presented an adversary-free framework for counterfactual prediction grounded in an information-theoretic risk bound that links the counterfactual–factual gap to mutual information. Building on this insight, we introduced SICE (static) and DICE (sequential), which learn stochastic representations that remain predictive for outcomes while directly penalizing treatment–representation dependence via a tractable variational surrogate, avoiding the instability of min–max training. Empirically, the approach is competitive across synthetic settings and a real clinical dataset, and it remains robust as treatment dimensionality grows, supporting complex, multi-attribute interventions.
Future work includes sensitivity-aware training and proximal/IV-style extensions, tighter MI estimators, removal of assignment bias in the case of unmeasured confoundings, and policy-learning objectives that leverage our representation for decision-making under sequential dynamics. \,\,




\bibliography{refs}

%
\runningtitle{Adversary-Free Counterfactual Prediction via Info-Regularization}

%

\onecolumn
\aistatstitle{Adversary-Free Counterfactual Prediction via Information-Regularized Representations: \\
	Supplementary Materials}

\paragraph{Standing notation.}
$X$ covariates; $T\in\mathcal T$ with law $\pi$; representation $Z\sim q_\phi(\cdot|X)$; outcome $Y$.
Per-arm Z-profile $\varphi_t(z):=\mathbb{E}[L(Y(t),g_t(z))| Z=z]$.
Factual and counterfactual risks
$R_F:=\mathbb{E}_{t\sim\pi}\mathbb{E}_{p(z| t)}[\varphi_t(z)]$ and
$R_{CF}:=\mathbb{E}_{t\sim\pi}\mathbb{E}_{p_Z}[\varphi_t(z)]$.
Total variation $\mathrm{TV}(\cdot,\cdot)$; Kullback–Leibler $ D_{\mathrm{KL}}(\cdot\|\cdot)$;
mutual information $I(Z;T)=\int D_{\mathrm{KL}}(p(z| t)\|p_Z)\,\pi(dt)$.

\section{Main Z-space bound}

\begin{theorem}[Z-space risk gap via MI]
	Assume there exists $\lambda>0$ and $\mathcal F \subset \{f:\|f\|_\infty \le 1\}$
	such that $\phi_t/\lambda \in \mathrm{conv}(\mathcal F)$ for all $t\in\mathcal T$.
	Then
	\[
	|R^{CF}-R^{F}| \le \sqrt{2}\,\lambda\,\sqrt{I(Z;T)}.
	\]
\end{theorem}

\begin{proof}
	Write $p_t(\cdot):=p(z\mid T=t)$. Then
	\[
	\begin{aligned}
		|R^{CF}-R^{F}|
		&= \left| \int \big(\mathbb E_{p_Z}[\phi_t]-\mathbb E_{p_t}[\phi_t]\big)\,\pi(dt)\right| \\
		&\le \int \left|\mathbb E_{p_Z}[\phi_t]-\mathbb E_{p_t}[\phi_t]\right|\,\pi(dt) \\
		&\le \lambda \int \mathrm{IPM}_{\mathcal F}(p_t,p_Z)\,\pi(dt) \\
		&\le 2\lambda \int \mathrm{TV}(p_t,p_Z)\,\pi(dt) \\
		&\le \sqrt{2}\,\lambda \int \sqrt{D_{\mathrm{KL}}(p_t\|p_Z)}\,\pi(dt) \\
		&\le \sqrt{2}\,\lambda \sqrt{\int D_{\mathrm{KL}}(p_t\|p_Z)\,\pi(dt)} \\
		&= \sqrt{2}\,\lambda\,\sqrt{I(Z;T)}.
	\end{aligned}
	\]
\end{proof}

\paragraph{Equality/tightness.}
Exact equality when $I(Z;T)=0$ (perfect balance).
Otherwise, the chain of inequalities is tight if simultaneously:
(i) all arm-wise TV maximizers align,
(ii) the divergences $ D_{\mathrm{KL}}(p(z| t)\|p_Z)$ are $\pi$-a.s.\ constant ,
(iii) $ D_{\mathrm{KL}}$ is small so that Pinsker is near-tight.

\section{Low MI limits treatment predictability}

\begin{proposition}[Binary case]
	Assume $T\in\{0,1\}$ and $\pi(0)=\pi(1)=1/2$. Let $e^\star$
	denote the Bayes error of predicting $T$ from $Z$. Then
	\[
	e^\star \ge \frac12 - \sqrt{\frac{I(Z;T)}{2}}.
	\]
\end{proposition}

\begin{proof}
	For equal priors,
	\[
	e^\star=\frac12\bigl(1-\mathrm{TV}(p_0,p_1)\bigr),
	\qquad p_i:=p(z\mid T=i).
	\]
	Also $p_Z=\tfrac12(p_0+p_1)$, so by triangle inequality,
	\[
	\mathrm{TV}(p_0,p_1)\le \mathrm{TV}(p_0,p_Z)+\mathrm{TV}(p_1,p_Z).
	\]
	Applying Pinsker's inequality and then Cauchy--Schwarz gives
	\[
	\mathrm{TV}(p_0,p_1)\le \sqrt{2\,I(Z;T)}.
	\]
	Substituting into the expression for $e^\star$ yields the claim.
\end{proof}

\begin{remark}[Multiclass uniform prior]
	If $T$ is uniform on $K\ge 2$ classes, then Fano's inequality yields
	\[
	e^\star \ge 1-\frac{I(Z;T)+\log 2}{\log K}.
	\]
\end{remark}

\section{Additional Experiments}

\FloatBarrier

We study how the mutual–information regularization
$\lambda\in\{10^{-5},10^{-4},10^{-3},10^{-2},0.1,1,10\}$
affects four metrics on a fixed synthetic dataset: test $\mathrm{RMSE}_y$, PEHE, ATE error, and HSIC$(z,t)$.
Let $z=f_\phi(x)$ denote the learned representation of $x$.
Following standard kernel–based independence testing \cite{Gretton2008HSIC}, we use HSIC$(z,t)$ as a dependence proxy between $z$ and the treatment $t$ (smaller indicates weaker $z$–$t$ coupling).

\setlength{\textfloatsep}{8pt}
\setlength{\floatsep}{8pt}
\setlength{\intextsep}{8pt}


\begin{figure}[!t]
	\centering
	\includegraphics[width=\columnwidth]{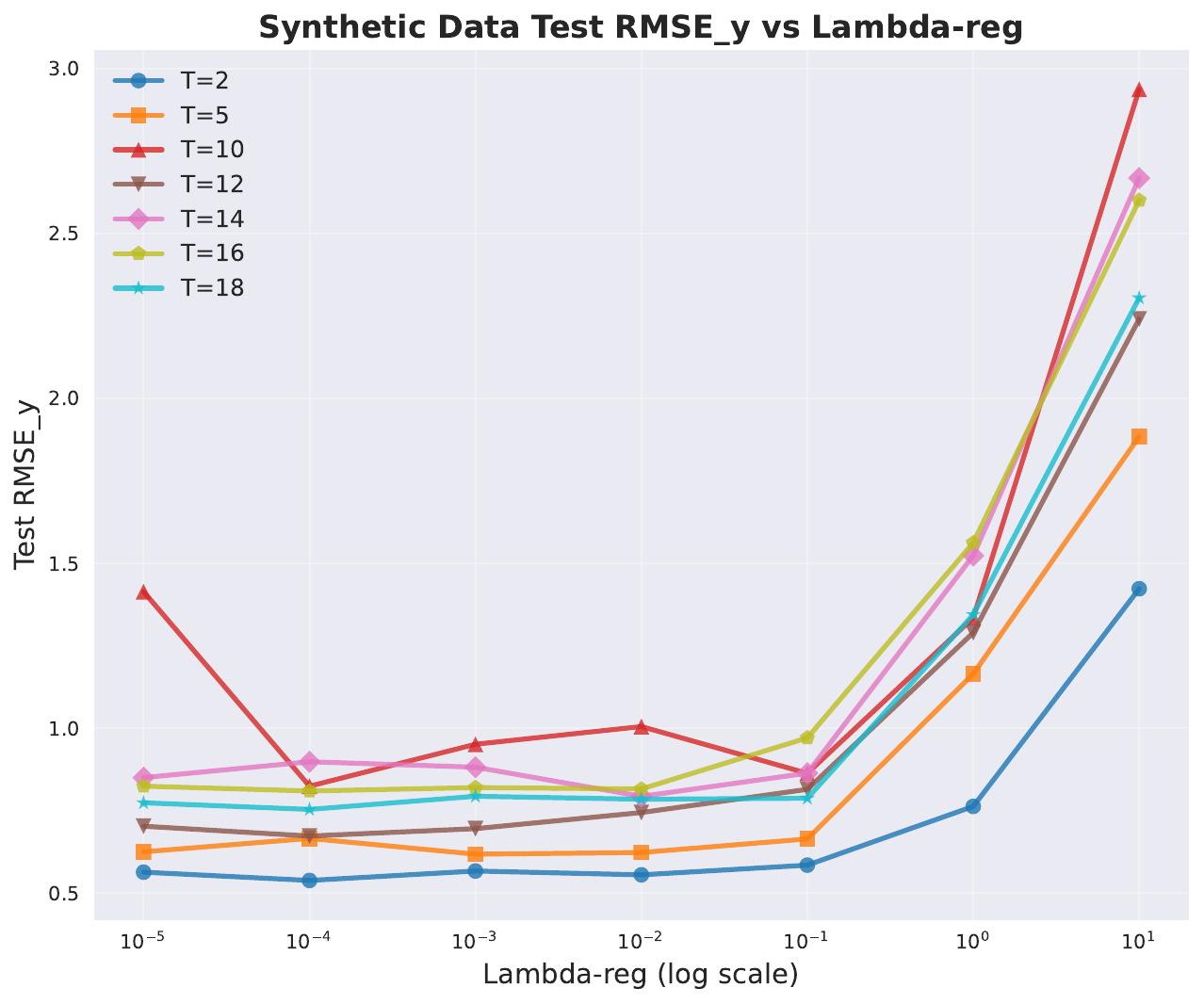}
	\caption{$\mathrm{RMSE}_y$ versus $\lambda\in\{10^{-5},10^{-4},10^{-3},10^{-2},0.1,1,10\}$ on the synthetic dataset (lower is better).}
	\label{fig:supp-rmse}
\end{figure}

\begin{figure}[!t]
	\centering
	\includegraphics[width=\columnwidth]{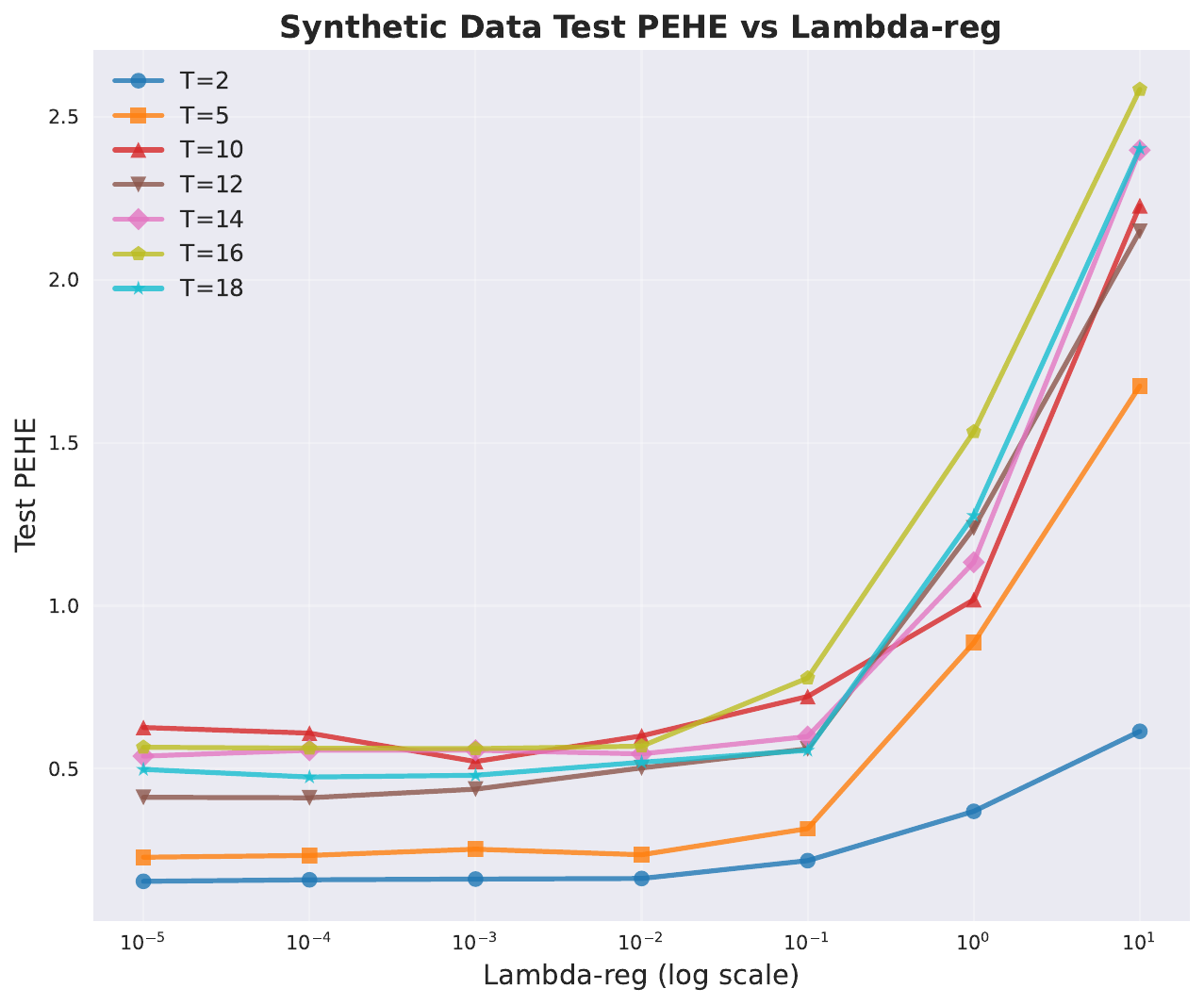}
	\caption{PEHE versus $\lambda$ on the synthetic dataset (lower is better).}
	\label{fig:supp-pehe}
\end{figure}

\begin{figure}[!t]
	\centering
	\includegraphics[width=\columnwidth]{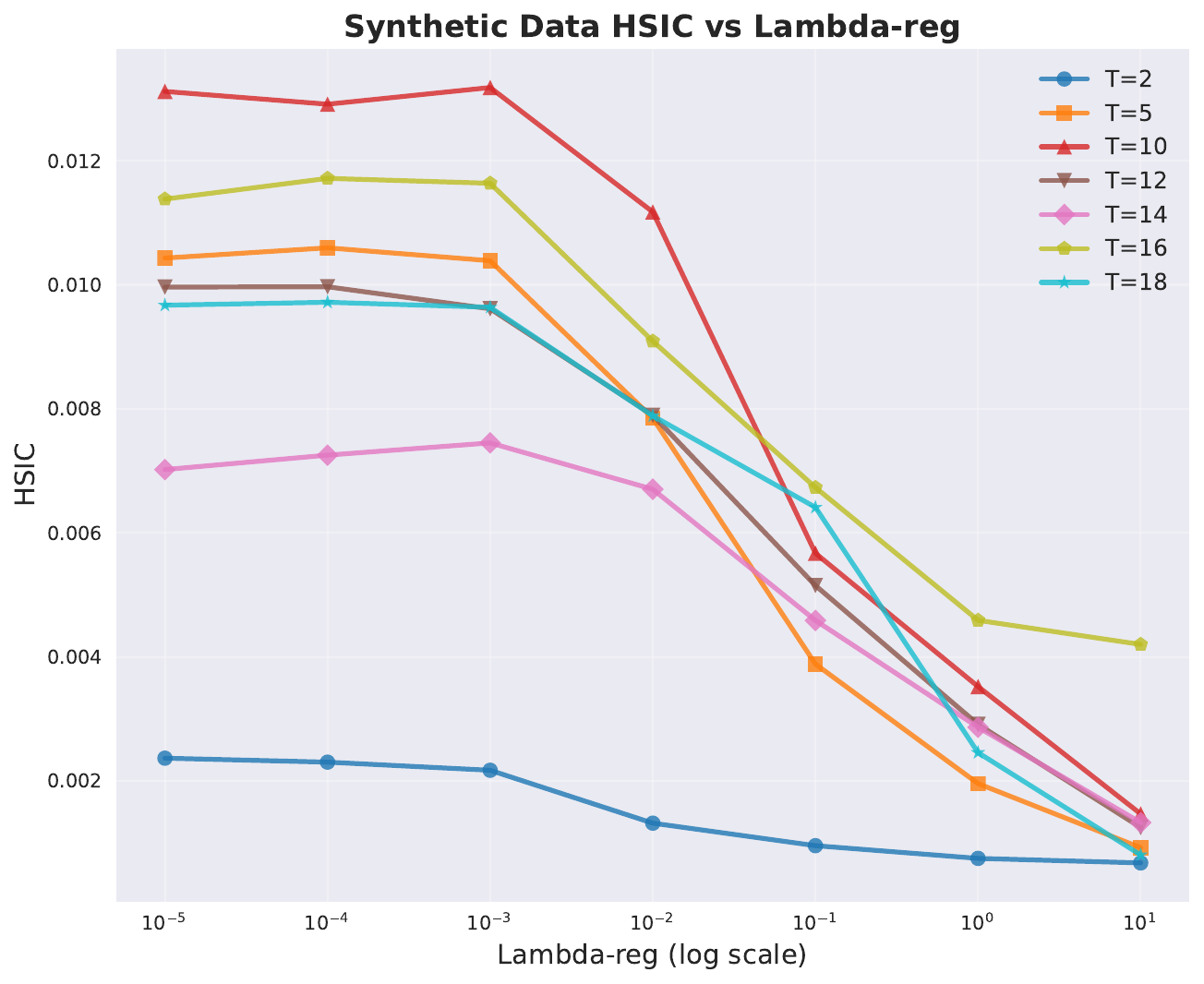}
	\caption{HSIC$(z,t)$ versus $\lambda$ as a kernel-based dependence proxy between $z=f_\phi(x)$ and the treatment $t$ (smaller indicates weaker dependence).}
	\label{fig:supp-hsic}
\end{figure}

\begin{figure}[!t]
	\centering
	\includegraphics[width=\columnwidth]{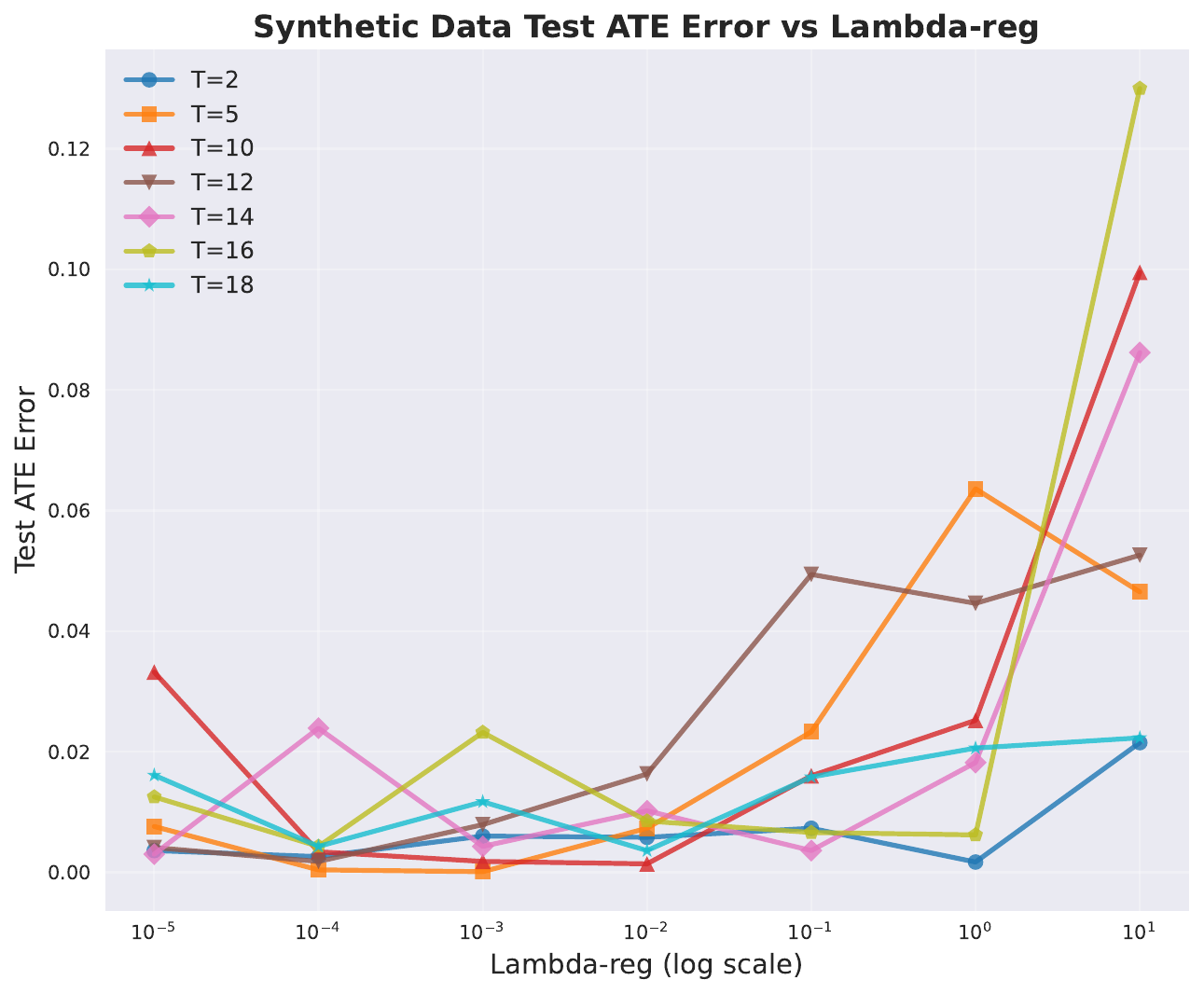}
	\caption{ATE error versus $\lambda$ on the synthetic dataset (lower is better).}
	\label{fig:supp-ate}
\end{figure}

\noindent\textbf{Findings.}
Across representation dimensions $t_{\mathrm{dim}}$, accuracy metrics favor small–to–moderate regularization:
$\mathrm{RMSE}_y$ and PEHE typically attain minima near $\lambda\!\approx\!10^{-4}\text{--}10^{-3}$, and ATE error is smallest for $\lambda\!\in[10^{-4},10^{-2}]$.
Larger regularization ($0.1,1,10$) consistently degrades accuracy, whereas HSIC$(z,t)$ decreases (monotonically or near-monotonically) with $\lambda$, indicating weaker $z$–$t$ coupling.
For example, at $t_{\mathrm{dim}}{=}2$, $\mathrm{RMSE}_y$ improves from $0.5639$ ($10^{-5}$) to $0.5388$ ($10^{-4}$) but rises to $1.4231$ at $\lambda{=}10$; PEHE is best at $10^{-3}$ ($0.1529$) yet increases to $0.6808$ at $\lambda{=}10$; HSIC drops from $2.98\!\times\!10^{-3}$ ($10^{-5}$) to $6.81\!\times\!10^{-4}$ ($10$).
At $t_{\mathrm{dim}}{=}10$, $\mathrm{RMSE}_y$ is $0.8246$ at $10^{-4}$ but $2.9369$ at $10$, while HSIC declines from $9.79\!\times\!10^{-3}$ ($10^{-5}$) to $1.09\!\times\!10^{-3}$ ($10$).
These results validate the intended role of the regularizer: \(\lambda\) acts as a knob on the shared information between the treatment \(t\) and the learned representation \(z=f_\phi(x)\).
As \(\lambda\) increases, the empirical dependence between \(z\) and \(t\)—estimated on the test set via HSIC as a kernel‐based proxy for \(I(z,t)\) \cite{Gretton2008HSIC}—consistently decreases across treatment dimensions.
Operationally, we train the model on the training split, embed test inputs to obtain \(z\), and then compute HSIC\((z,t)\); the near‐monotone decline of HSIC with larger \(\lambda\) demonstrates that stronger regularization yields representations \(z\) that carry less information about \(t\).

\begin{table}[!h]
	\centering
	\captionsetup{skip=2pt}
	\caption{Results across $t_{\mathrm{dim}}$ and $\lambda$ (averaged over runs; standard deviations omitted for space).
		Rows are $t_{\mathrm{dim}}$; columns are $\lambda\in\{10^{-5},10^{-4},10^{-3},10^{-2},0.1,1,10\}$.}
	\label{tab:sice_ablation_stacked}
	
	\scriptsize
	\setlength{\tabcolsep}{3pt}
	\renewcommand{\arraystretch}{0.95}
	
	\textbf{(a) $\mathrm{RMSE}_y$}\\[-0.2ex]
	\begin{tabularx}{\linewidth}{@{}lYYYYYYY@{}}
		\toprule
		$t_{\mathrm{dim}}$ & $10^{-5}$ & $10^{-4}$ & $10^{-3}$ & $10^{-2}$ & $0.1$ & $1$ & $10$ \\
		\midrule
		2  & 0.563900 & 0.538800 & 0.567400 & 0.555800 & 0.585200 & 0.763300 & 1.423100 \\
		5  & 0.625600 & 0.666200 & 0.618400 & 0.623500 & 0.664700 & 1.165100 & 1.884000 \\
		10 & 1.414200 & 0.824600 & 0.951600 & 1.005300 & 0.863600 & 1.335200 & 2.936900 \\
		12 & 0.703200 & 0.673400 & 0.695800 & 0.744600 & 0.815000 & 1.289800 & 2.240700 \\
		14 & 0.850500 & 0.898600 & 0.882100 & 0.793700 & 0.864000 & 1.523500 & 2.668000 \\
		16 & 0.824600 & 0.810000 & 0.820300 & 0.816100 & 0.970900 & 1.562300 & 2.600400 \\
		18 & 0.774200 & 0.754100 & 0.794000 & 0.785000 & 0.788000 & 1.344500 & 2.304600 \\
		\bottomrule
	\end{tabularx}
	
	\vspace{0.6ex}
	
	\textbf{(b) PEHE}\\[-0.2ex]
	\begin{tabularx}{\linewidth}{@{}lYYYYYYY@{}}
		\toprule
		$t_{\mathrm{dim}}$ & $10^{-5}$ & $10^{-4}$ & $10^{-3}$ & $10^{-2}$ & $0.1$ & $1$ & $10$ \\
		\midrule
		2  & 0.154300 & 0.159000 & 0.152900 & 0.155300 & 0.167500 & 0.250400 & 0.680800 \\
		5  & 0.276600 & 0.286900 & 0.277000 & 0.284100 & 0.349200 & 0.833200 & 1.369800 \\
		10 & 0.707000 & 0.555600 & 0.548500 & 0.561900 & 0.612500 & 0.939200 & 2.183300 \\
		12 & 0.468800 & 0.462500 & 0.468700 & 0.486200 & 0.539900 & 0.914100 & 1.862400 \\
		14 & 0.598900 & 0.639700 & 0.562600 & 0.581600 & 0.601300 & 1.096400 & 2.221100 \\
		16 & 0.703800 & 0.675100 & 0.695600 & 0.708400 & 0.730900 & 1.188400 & 2.359900 \\
		18 & 0.736200 & 0.724800 & 0.736700 & 0.737300 & 0.740400 & 1.162100 & 2.226800 \\
		\bottomrule
	\end{tabularx}
	
	\vspace{0.6ex}
	
	\textbf{(c) ATE error}\\[-0.2ex]
	\begin{tabularx}{\linewidth}{@{}lYYYYYYY@{}}
		\toprule
		$t_{\mathrm{dim}}$ & $10^{-5}$ & $10^{-4}$ & $10^{-3}$ & $10^{-2}$ & $0.1$ & $1$ & $10$ \\
		\midrule
		2  & 0.003700 & 0.000100 & 0.000500 & 0.000200 & 0.000800 & 0.005100 & 0.048200 \\
		5  & 0.002200 & 0.002600 & 0.000100 & 0.000100 & 0.000900 & 0.013400 & 0.090000 \\
		10 & 0.053600 & 0.013700 & 0.019000 & 0.006000 & 0.009800 & 0.032800 & 0.203200 \\
		12 & 0.009200 & 0.009400 & 0.007200 & 0.003700 & 0.009200 & 0.047700 & 0.206200 \\
		14 & 0.021700 & 0.028800 & 0.015700 & 0.013200 & 0.014000 & 0.091000 & 0.302500 \\
		16 & 0.019700 & 0.019800 & 0.018800 & 0.013500 & 0.015600 & 0.079600 & 0.316400 \\
		18 & 0.019000 & 0.019300 & 0.016700 & 0.015000 & 0.014900 & 0.079500 & 0.303300 \\
		\bottomrule
	\end{tabularx}
	
	\vspace{0.6ex}
	
	\textbf{(d) HSIC$(z,t)$}\\[-0.2ex]
	\begin{tabularx}{\linewidth}{@{}lYYYYYYY@{}}
		\toprule
		$t_{\mathrm{dim}}$ & $10^{-5}$ & $10^{-4}$ & $10^{-3}$ & $10^{-2}$ & $0.1$ & $1$ & $10$ \\
		\midrule
		2  & 0.002981 & 0.002975 & 0.002821 & 0.002678 & 0.001845 & 0.001208 & 0.000681 \\
		5  & 0.004485 & 0.004276 & 0.003993 & 0.003615 & 0.001637 & 0.000781 & 0.000603 \\
		10 & 0.009785 & 0.009740 & 0.008795 & 0.007895 & 0.005640 & 0.001608 & 0.001089 \\
		12 & 0.004377 & 0.004383 & 0.004090 & 0.003659 & 0.003045 & 0.001790 & 0.000969 \\
		14 & 0.008351 & 0.008332 & 0.007280 & 0.006456 & 0.004819 & 0.001721 & 0.001290 \\
		16 & 0.005608 & 0.005535 & 0.004826 & 0.004629 & 0.003576 & 0.002113 & 0.003355 \\
		18 & 0.006674 & 0.006705 & 0.006611 & 0.005721 & 0.005841 & 0.002936 & 0.005055 \\
		\bottomrule
	\end{tabularx}
	
\end{table}

\FloatBarrier

\end{document}